\documentclass[conference]{IEEEtran}
\IEEEoverridecommandlockouts
\usepackage{cite}
\usepackage{amsmath,amssymb,amsfonts}
\usepackage{balance}
\usepackage{algorithmic}
\usepackage{graphicx}
\usepackage{textcomp}
\usepackage{xcolor}
\usepackage{amsmath}
\usepackage{amssymb}
\usepackage[mathscr]{euscript}
\usepackage{comment}
\usepackage{xcolor}
\usepackage{enumitem} 
\usepackage{amsthm}
\usepackage{hyperref}
\usepackage{cleveref}
\usepackage{mathtools}
\newtheorem{theorem}{Theorem}
\newtheorem{corollary}{Corollary}
\newtheorem{definition}{Definition}
\newtheorem{proposition}{Proposition}

\def\BibTeX{{\rm B\kern-.05em{\sc i\kern-.025em b}\kern-.08em
    T\kern-.1667em\lower.7ex\hbox{E}\kern-.125emX}}
\begin{document}

\title{{Information-Theoretic Minimax Regret Bounds for Reinforcement Learning based on Duality}\\

\thanks{This work is supported by the Knut and Wallenberg Foundation.}
}

\author{
\IEEEauthorblockN{Raghav Bongole, Amaury Gouverneur, Borja Rodríguez-Gálvez, Tobias J. Oechtering, and Mikael Skoglund}
\IEEEauthorblockA{Division of Information Science and Engineering (ISE)\\
KTH Royal Institute of Technology\\
\texttt{\{bongole, amauryg, borjarg, oech, skoglund\}@kth.se}}
}

\maketitle

\begin{abstract}
We study agents acting in an unknown environment where the agent's goal is to find a robust policy. We consider robust policies as policies that achieve high cumulative rewards for all possible environments. To this end, we consider agents minimizing the maximum regret over different environment parameters, leading to the study of minimax regret. This research focuses on deriving information-theoretic bounds for minimax regret in Markov Decision Processes (MDPs) with a finite time horizon. Building on concepts from supervised learning, such as minimum excess risk (MER) and minimax excess risk, we use recent bounds on the Bayesian regret to derive minimax regret bounds. Specifically, we establish minimax theorems and use bounds on the Bayesian regret to perform minimax regret analysis using these minimax theorems. Our contributions include defining a suitable minimax regret in the context of MDPs, finding information-theoretic bounds for it, and applying these bounds in various scenarios.
\end{abstract}


\section{Introduction}
\label{sec:introduction}
The study of reinforcement learning (RL) in adversarial environments has received significant attention in recent years \cite{moos2022robust}. Traditional RL approaches assume a stochastic environment where the transition probabilities and rewards are fixed but unknown. A common goal in traditional RL is to find a policy that minimizes regret, which is the difference between the optimal expected cumulative reward that could be achieved if the environment was known and the expected cumulative reward achieved by the agent. 

In many scenarios, we are interested in finding a robust policy that performs well for all possible environments. 
We consider an agent that wants to find a robust policy when the environment acts as an adversary, selecting the worst-case environment parameters. For a given class of Markov Decision Processes, the goal of the agent is to choose a (possibly randomized) policy that tries to minimize the regret for the worst case. This regret is referred to as minimax regret. While theoretical bounds for the minimax regret in some settings are established \cite{lattimore2020bandit, azar2017minimax}, its information-theoretic nature has remained largely unexplored. This research establishes information-theoretic bounds for minimax regret. 

In supervised learning, a notable concept is the minimum excess risk (MER) \cite{gyorfi2023lossless}, \cite{hafez2021rate}. MER is an algorithm-independent quantity that measures, for a supervised learning problem, the gap between its fundamental limit and the best possible algorithm that is uncertain about the environment but has access to the data. MER quantifies the gap between the best possible loss incurred when the algorithm has knowledge of the environment parameters and when it only has access to the data. The work by Xu and Raginsky \cite{xu2022minimum} establishes information-theoretic bounds for the MER in various settings.

Another studied notion in supervised learning is the minimax excess risk, as discussed by \cite{hafez2023information}. In \cite{hafez2023information}, the authors employ minimax theorems to identify conditions under which the minimax duality holds, that is when $\max_x \min_y f(x,y)$ equals $\min_y \max_x f(x,y)$ for a given risk function $f(x,y)$. Then, the authors use bounds established by  \cite{xu2022minimum} for MER to bound the minimax excess risk.

The concept of minimax regret has been investigated in various settings within the literature. Lattimore and Szepesvári explore this topic in the context of partial monitoring, deriving information-theoretic bounds in the Bayesian setting and establishing minimax theorems \cite{lattimore2019information}. Then they use these   bounds on the Bayesian regret to subsequently bound  the minimax regret in finite-action partial monitoring settings. In the context of Markov Decision Processes (MDPs), Buening et al. \cite{buening2023minimax} establish a minimax theorem for finite state and action spaces, assuming that the reward is a known deterministic function of the state and action. In this paper, we extend their results and prove a minimax theorem for continuous states and action spaces given additional conditions such as continuity of regret.
Gouverneur et al. \cite{gouverneur2022information} establish bounds for the minimum Bayesian regret (MBR) in reinforcement learning. Like the MER, the MBR is an algorithm-independent quantity that measures the gap between the fundamental limit of an RL problem and the best possible cumulative reward that an agent can achieve.

Our research builds on these foundations by focusing on the RL problem in the context of MDPs with finite time horizon. Inspired by \cite{xu2022minimum, hafez2023information, lattimore2019information, buening2023minimax, gouverneur2022information}, we derive minimax theorems and use the bounds on the minimum Bayesian regret to establish information-theoretic minimax regret bounds. More specifically, we make the following contributions. (i) We define a notion of minimax regret for RL that is suitable for information-theoretic analysis. (ii) Using duality principles, we then establish connections between the minimax regret and the minimum Bayesian regret. (iii) We then derive bounds on minimax regret, demonstrating their applicability in various scenarios. (iv) Finally, we derive explicit minimax regret bounds for various settings, such as multi-armed bandits, linear bandits, and contextual bandits.

\section{Notation and Preliminaries}
\label{sec:notation}
Sets are denoted by calligraphic letters (e.g., state space \(\mathcal{S}\)), various notions of regret by fraktur font (e.g., \(\mathfrak{R}\)) and $\sigma$-algebras are denoted by script letters (e.g., \(\mathscr{S}\)). 

 Random variables are denoted by capital letters (e.g., \(\Theta\)), and their realizations by lowercase letters (e.g., \(\theta\)). The expectation of $X$ is denoted by \(\mathbb{E}_X[X]\) (or $\mathbb{E}[X]$ if it is clear from context), and conditional expectation of $X$ given $y$ is given by \(\mathbb{E}^y[X]\). $\Delta(\mathcal{X})$ denotes the set of all probability measures on $\mathcal{X}$. Entropy \(H(X)\), KL-divergence \(D_{KL}(P\mid\mid Q)\), and mutual information \(I(X; Y)\) are defined in standard terms.
 Consider a Polish space \(\mathcal{X}\) equipped with a metric \(\rho\). For probability distributions \(P\) and \(Q\) on \(\mathcal{X}\), the Wasserstein distance is defined as \(W(P,Q) \coloneqq \inf_{D \in \Gamma(P,Q)} \int \rho \, dD\), where \(\Gamma(P,Q)\) is the set of joint distributions with marginals \(P\) and \(Q\).

\section{Model and Definitions}
\label{sec:model}
In the context of reinforcement learning, we consider an agent navigating through an uncertain environment, often modeled as a Markov Decision Process (MDP)\cite{gouverneur2022information} with a finite time horizon  $T \in \mathbb{N}$ time steps. At each time step $t \in \{1, \ldots, T\}$, the environment is characterized by a state $S_t \in \mathcal{S}$ and the agent selects an action $A_t \in \mathcal{A}$. The state then transitions to $S_{t+1}$, and the environment produces an outcome $Y_t \in \mathcal{Y}$ that the agent associates with a reward $R_t$.

More formally, we define a class of MDPs $\mathcal{M}$, parameterized by a random variable $\Theta \in \mathcal{O}$. This class is defined by a state space $\mathcal{S}$, an action space $\mathcal{A}$, a transition kernel $ p: \mathscr{S} \times (\mathcal{S} \times \mathcal{A} \times \mathcal{O}) \rightarrow [0, 1]$ such that
$
\mathbb{P}_{S_{t+1}|S_t,A_t,\Theta} = p(\cdot, (S_t, A_t, \Theta)),
$
and an outcome kernel $y : \mathscr{Y} \times (\mathcal{S} \times  \mathcal{O}) \rightarrow [0, 1]$ such that
$
\mathbb{P}_{Y_t|S_t,\Theta} = y(\cdot, (S_t, \Theta)),
$
and an initial state prior distribution $\mathbb{P}_{S_1|\Theta}$ such that $S_1 \sim \mathbb{P}_{S_1|\Theta}$. The reward received by the agent is modeled as a deterministic function of the outcome \( Y_t \) and the chosen action \( A_t \), denoted by \( r(Y_t, A_t) \). We can write the random variable \( R_t \) as \( R_t = r(Y_t, A_t) \). Thus the class of MDPs $\mathcal{M}$ can be seen as a 6-tuple $\mathcal{M}\coloneqq (\mathcal{S}, \mathcal{A}, p, y, r, T)$.

For a fixed $\theta$, we define $\mathcal{M}_{\theta}$ as the MDP corresponding to $\theta$. The value of $\theta$ is unknown to the agent, however, $\mathcal{S}, \mathcal{A}, p,y, r$ and $T$ are known. Learning involves characterizing the environment parameter $\theta$. In the Bayesian setting, the agent is assumed to have access to a known prior \( \mathbb{P}_\Theta \). However, our primary interest lies in the frequentist setting, where the environment parameter \( \theta \) is unknown, and the goal is to characterize bounds for the worst-case parameter \( \theta \). We leverage Bayesian reinforcement learning bounds to derive these frequentist bounds.

Let the history up to time $t$ be a sequence of random variables denoted by $H^t \in \mathcal{H}^t$, where $H^t = (H_1, \ldots, H_{t})$ and $H_{t+1} = (S_t, A_t, Y_t)$. The agent selects actions based on its current state and history, formalized through a policy function \(\pi\coloneqq \{\pi_t: \mathcal{S} \times \mathcal{H}^t \rightarrow \mathcal{A}\}_{t=1}^T\). Further, let $\mathcal{P}$ be the set of all policies.

We define the utility, quantifying the cumulative reward the agent expects to achieve by following a given policy $\pi$ for a given MDP $\mathcal{M}_\theta$.

\begin{definition}[Utility]
The utility of an agent in an MDP $\mathcal{M}_{\theta}$ following a policy $\pi$ is given by
$
U_{\mathcal{M}}(\pi_t, \theta) \coloneqq \mathbb{E}^\theta\left[ \sum_{t=1}^T r(Y_t, \pi_t(S_t, H^t)) \right],
$
where the expectation is taken over the random variables $Y_t$, $S_t$ and $H^t$ for $t \in \{1,..T\}$ drawn following policy $\pi$.
\end{definition}
To evaluate the efficiency of a policy, we can compare it with the optimal utility, which represents the best possible outcome if the agent has perfect knowledge of the environment parameter \(\theta\).
\begin{definition}[Optimal Utility]
The optimal utility of an agent in an MDP $\mathcal{M}_{\theta}$ is the maximum utility that can be attained when the policy function has access to the environment parameters. This quantity is algorithm-independent and is given by 
$
   U^*_{\mathcal{M}}(\theta)\coloneqq \sup_{f_{\theta}: \mathcal{S} \rightarrow \mathcal{A} }\mathbb{E}^\theta\left[ \sum_{t=1}^T r(Y_t, f_{\theta}(S_t))  \right]. 
$
\end{definition}
The agent's performance is assessed by the regret, which is defined as the difference between the optimal utility and the utility achieved by the agent.

\begin{definition}[Regret]
The regret of the agent is the difference between the optimal utility of an MDP and the utility under a particular policy $\pi$. It is given by
$\mathfrak{R}_{\mathcal{M}}(\pi, \theta)\coloneqq U^*_{\mathcal{M}}( \theta) - U_{\mathcal{M}}(\pi, \theta).$
\end{definition}
Our goal is to characterize the optimal worst-case regret, defined as the minimax regret. It evaluates the policy's performance under the most challenging environment parameter.
\begin{definition}[Minimax Regret]
The minimax regret, $\mathfrak{M}_{\mathcal{M}}$, is the regret obtained by an agent following the optimal policy for the worst-case environment parameters. The minimax regret is given by $\mathfrak{M}_{\mathcal{M}} \coloneqq \inf_{\mathbb{P}_\Pi \in \Delta(\mathcal{P})} \sup_{\theta \in \mathcal{O}} \mathbb{E}_\Pi \left[\mathfrak{R}_{\mathcal{M}}(\Pi, \theta) \right].$

\end{definition}
In contrast to minimax regret, Bayesian regret considers a prior distribution over environment parameters, providing an average-case analysis.

\begin{definition}[Bayesian Regret]
The Bayesian regret of an MDP following a policy $\pi$, denoted by $\mathfrak{BR}_{\mathcal{M}}(\pi, \mathbb{P}_\Theta)$ is defined as the average regret with respect to a prior distribution of $\mathbb{P}_\Theta$ of random variable $\Theta$. The Bayesian Regret is given by 
$\mathfrak{BR}_{\mathcal{M}}(\pi, \mathbb{P}_\Theta)\coloneqq 
\mathbb{E}_{\Theta \sim \mathbb{P}_\Theta}[\mathfrak{R}_{\mathcal{M}}(\pi, \Theta)].$
\end{definition}
\begin{proposition}
Let $\mathbb{P}_\Theta$ be absolutely continuous with respect to $\mu$ with density $p_\Theta$. Then, the Bayesian regret can be expressed as 
\begin{align*}
\mathfrak{BR}_{\mathcal{M}}(\pi, \mathbb{P}_\Theta)\smash=&\sup \mathbb{E}\left[ \sum_{t=1}^T r(Y_t, f(S_t,\Theta)) \right] - \mathbb{E}_\Theta \left[U_{\mathcal{M}}(\pi, \Theta)\right]
\end{align*}
where the supremum is taken over the function space ${f: \mathcal{S} \times \mathcal{O} \rightarrow \mathcal{A}}  $.
\end{proposition}

\begin{proof} By using the linearity of expectation, we can interchange the supremum and the expectation, leading to the following equalities for the Bayesian regret:
\begin{align*}
    &\mathfrak{BR}_{\mathcal{M}}(\pi, \mathbb{P}_\Theta)= 
\mathbb{E}_{\Theta \sim \mathbb{P}_\Theta}[\mathfrak{R}_{\mathcal{M}}(\pi, \Theta)] = \int_{\theta} \mathfrak{R}(\pi, \theta)p_\Theta(\theta) d\mu \\
&= \int_{\theta}\sup_{f: \mathcal{S} \times \mathcal{O} \rightarrow \mathcal{A}} V(f, \theta) p_\Theta(\theta) d\mu 
-\int_{\theta}U_{\mathcal{M}}(\pi, \theta)p_\Theta(\theta) d\mu  \\
&= \sup_{f: \mathcal{S} \times \mathcal{O} \rightarrow \mathcal{A}} \mathbb{E}\left[ \sum_{t=1}^T r(Y_t, f(S_t,\Theta)) \right] - \mathbb{E}_\Theta \left[U_{\mathcal{M}}(\pi, \Theta)\right].
\end{align*}
where ${V}(f, \theta)\smash{\coloneqq} \mathbb{E}^\theta\big[ \sum_{t=1}^T r(Y_t, f(S_t,\theta)) \big]$. 
\end{proof}

Further, we define the minimum Bayesian regret as the lowest achievable average regret under any policy, given the prior distribution over the environment parameters.

\begin{definition}[Minimum Bayesian Regret]
Let us consider $\pi\coloneqq \{\pi_t : \mathcal{S} \times \mathcal{H}^t \rightarrow \mathcal{A}\}^T_{t=1}$. The minimum Bayesian regret (MBR) under a prior $\mathbb{P}_\Theta$ is denoted by $\mathfrak{F}_{\mathcal{M}}(\mathbb{P}_\Theta)$ and is given by
$\mathfrak{F}_{\mathcal{M}}(\mathbb{P}_\Theta)\coloneqq \inf_{\pi}\mathfrak{BR}_{\mathcal{M}}(\pi, \mathbb{P}_\Theta).
$

\end{definition}
Finally, we introduce the worst-case minimum Bayesian regret, corresponding to the maximum MBR achievable under any prior. This concept is critical in game-theoretic scenarios, where the environment selects the most challenging prior, and the agent optimizes its policy in response.

\begin{definition}[Worst-case MBR]
The Worst-case MBR is the maximum MBR that can be achieved by any prior. It can also be interpreted in a game-theoretic fashion as the Bayesian regret when the environment plays first, choosing the worst prior, and the agent plays next, selecting the best possible policy for that prior. Hence, the worst-case MBR is given by
$
\begin{aligned}[t]
&\mathfrak{F}^*_{\mathcal{M}}\coloneqq \sup_{\mathbb{P}_\Theta \in \Delta(\mathcal{O})} \mathfrak{F}_{\mathcal{M}}(\mathbb{P}_\Theta). 
\end{aligned}
$
\end{definition}

\section{Minimax Theorems}
\label{sec:minimax}
In this section, we will state and prove minimax theorems that provide conditions where the worst-case minimum Bayesian regret and the minimax regret are equal.

\subsection{Minimax theorem with continuity conditions}

We will adapt the minimax theorem from Cesa-Bianchi et al. \cite[Theorem 7.1]{cesa2006prediction}. To prove the minimax theorem for regret under these conditions, we need the following to hold:
\begin{enumerate}[label=(\roman*)]
    \item $\mathbb{E}_\Pi \mathbb{E}_\Theta  \left[\mathfrak{R}_{\mathcal{M}}(\Pi, \Theta) \right]$ is bounded and real valued.
    \item $\Delta(\mathcal{P})$ and $\Delta(\mathcal{O})$ are convex sets.
    \item $\Delta(\mathcal{O})$ is in addition, a compact set.
    \item $\mathbb{E}_\Pi \mathbb{E}_\Theta  \left[\mathfrak{R}_{\mathcal{M}}(\Pi, .) \right]$ is continuous with respect to $\mathbb{P}_\Theta$ for a fixed $\mathbb{P}_\Pi$.
    \item $\mathbb{E}_\Pi \mathbb{E}_\Theta  \left[\mathfrak{R}_{\mathcal{M}}(\Pi, \Theta) \right]$ is convex for fixed $\mathbb{P}_{\Pi}$ and is concave for fixed $\mathbb{P}_{\Theta}$. 
\end{enumerate}
\begin{theorem}\label{theorem: minimax for continuous}
    \textit{Let $(\mathcal{O},\delta_\mathcal{O})$ be a compact metric space and $(\pi,  \delta_\mathcal{P})$ be a metric space of policies. Let $ \Delta(\mathcal{O})$ and $  \Delta(\mathcal{P})$ be the set of all Borel probability measures on $(\mathcal{O},\delta_\mathcal{O})$ and $(\pi, \delta_\mathcal{P})$ respectively. Under the conditions $\mathfrak{R}_{\mathcal{M}}(\pi, \theta)$ is bounded for all $\pi \in \mathcal{P}$ and $\theta \in \mathcal{O}$,  and  if $\mathfrak{R}_{\mathcal{M}}(\pi, .)$ is a continuous function of $\theta$ for each $\pi$ we have, \(\mathfrak{M}_{\mathcal{M}} = \mathfrak{F}^*_{\mathcal{M}}.\)}
\end{theorem}
\begin{proof}
The proof follows from Theorem 2 in \cite{hafez2023information}, where they establish minimax duality conditions for excess risk in supervised learning. To prove the minimax theorem in our setting, we verify conditions (i) through (v).
 Condition (i) holds by assumption.
Condition (ii) is satisfied because the set of all Borel probability measures on a metric space is known to form a convex set.
For condition (iii), we observe that \(\Delta(\mathcal{O})\) is compact with respect to the Prokhorov metric \(\delta_m\) \cite[Proposition 5.3] {van2003probability}.
Condition (iv) is verified by demonstrating sequential continuity, which is equivalent to continuity in a metric space \cite{orevkov1973equivalence}. Let us define $g(\theta)$ for any fixed $\theta$ as \( g(\theta) = \mathfrak{R}_{\mathcal{M}}(\pi, \theta) \). We know $g(\theta)$ is both bounded and continuous. If \( \mathbb{P}_{\Theta_n}\) converges to \( \mathbb{P}_{\Theta_0} \) under the Prokhorov metric \(\delta_m\), then \( \mathbb{P}_{\Theta_n} \) converges to \( \mathbb{P}_{\Theta_0} \) weakly in measure, as \((\mathcal{O}, \delta_\mathcal{O})\) is a compact and separable metric space \cite[Theorem 4.2]{van2003probability}. We define \( h(\mathbb{P}_{\Theta}) \) as the expected value of \( g(\Theta) \) under the distribution \( \mathbb{P}_{\Theta} \), i.e., 
\(h(\mathbb{P}_{\Theta}) = \mathbb{E}_\Theta[g(\Theta)] = \int g(\theta) \, d\mathbb{P}_{\Theta}.\)
Furthermore, let us define the sequence \(\{h_n\}_{n\in \mathbb{N}}\), where \( h_n = h(\mathbb{P}_{\Theta_n}) \). Convergence in measure, combined with the fact that \(g(\theta)\) is bounded and continuous, implies that \( \int g(\theta) \, d\mathbb{P}_{\Theta_n}\) converges to \( \int g(\theta) \, d\mathbb{P}_{\Theta_0} \) pointwise. Hence, we conclude that \( h_n \) converges to \(  h_0 \) \cite[Theorem 3.2]{van2003probability}). The Dominated Convergence Theorem, along with the boundedness and pointwise convergence, ensures that \( \lim_{n \to \infty} \mathbb{E}_\Pi [h_n] = \mathbb{E}_\Pi [h_0] \) \cite{gray2009probability}, confirming the continuity of \( \mathbb{E}_\Pi \mathbb{E}_\Theta [\mathfrak{R}_{\mathcal{M}}(\Pi, \cdot)] \) with respect to \( \mathbb{P}_\Theta \).
Condition (v) is satisfied because the expectation is linear with respect to the distribution.

\end{proof}
\Cref{theorem: minimax for continuous} establishes that the minimax duality holds even in the case of continuous state and action spaces provided some additional conditions on regret are met. Further, using~\Cref{theorem: minimax for continuous}, we demonstrate that the minimax theorem holds even in the case of stochastic rewards with finite states and actions and with a finite class of MDPs. Also, notably, $\mathcal{Y}$ and $\mathcal{H}$ need not be finite.
\begin{corollary}
\label{corollary: minimax for finite}
If \( \mathfrak{R}_{\mathcal{M}}(\pi, \theta) \) is bounded for all \( \pi \in \mathcal{P} \) and \( \theta \in \mathcal{O} \), and if \( \mathcal{S} \), \( \mathcal{A} \), and \( \mathcal{O} \) are finite, then \( \mathfrak{M}_{\mathcal{M}} = \mathfrak{F}^*_{\mathcal{M}} \).
\end{corollary}
\begin{proof}
Let us equip $\mathcal{P}$ and $\mathcal{O}$ with the discrete metrics $\delta_\mathcal{P}$ and $\delta_\mathcal{O}$, respectively. Since \( \mathcal{O} \) is finite, the space \( (\mathcal{O}, \delta_\mathcal{O}) \) is compact under the discrete metric \( \delta_\mathcal{O} \). Additionally, \( \mathfrak{R}_\mathcal{M}(\pi, \cdot) \) is continuous with respect to \( \theta \) for a fixed \( \pi \). This follows from the fact that, under the discrete metric, any convergent sequence \(\{\theta_n\}_{n \in \mathbb{N}} \subset \mathcal{O}\) eventually becomes constant, implying that \( \mathfrak{R}_\mathcal{M}(\pi, \theta_n) \) converges. Therefore, all the assumptions of~\Cref{theorem: minimax for continuous} are satisfied, and we conclude that \( \mathfrak{M}_{\mathcal{M}} = \mathfrak{F}^*_{\mathcal{M}} \).

\end{proof}

\section{Upper bounds on minimax regret}
\label{sec:upper_bounds}
Under certain conditions, the minimum Bayesian regret can be bounded above by a quantity dependent on the prior \cite{gouverneur2022information}. Specifically, we have the following inequalities: \(\mathfrak{F}_{\mathcal{M}}(\mathbb{P}_\Theta) =    
 \inf_{\pi \in \mathcal{P}} \mathbb{E}_\Theta \left[ \mathfrak{R}_{\mathcal{M}}(\pi, \Theta) \right] = 
\inf_{\mathbb{P}_\Pi \in \Delta(\mathcal{P})} \mathbb{E}_\Pi \mathbb{E}_\Theta  \left[\mathfrak{R}_{\mathcal{M}}(\Pi, \Theta) \right]  
 \leq K_1(\mathbb{P}_\Theta),\)
where \(K_1(\mathbb{P}_\Theta)\) is a quantity determined by the prior \(\mathbb{P}_\Theta\). Hence, we can bound the worst-case MBR as follows:

$
\mathfrak{F}^*_{\mathcal{M}} =
\sup_{\mathbb{P}_\Theta \in \Delta(\mathcal{O})}\inf_{\mathbb{P}_\Pi \in \Delta(\mathcal{P})} \mathbb{E}_\Theta \mathbb{E}_\Pi \left[\mathfrak{R}_{\mathcal{M}}(\Pi, \Theta)\right] \leq K_2,
$
where $K_2$ is a quantity independent of the prior.
Furthermore, when the conditions of the minimax theorem are satisfied, the minimax  regret itself can be similarly bounded:
\[
\begin{split}
\mathfrak{M}_{\mathcal{M}} 
&= \inf_{\mathbb{P}_\Pi \in \Delta(\mathcal{P})} \sup_{\theta \in \mathcal{O}} \mathbb{E}_\Pi \left[\mathfrak{R}_{\mathcal{M}}(\Pi, \theta) \right] \\
&= \inf_{\mathbb{P}_\Pi \in \Delta(\mathcal{P})} \sup_{\mathbb{P}_\Theta \in \Delta(\mathcal{O})} \mathbb{E}_\Theta \mathbb{E}_\Pi \left[\mathfrak{R}_{\mathcal{M}}(\Pi, \Theta) \right] \\
&= \sup_{\mathbb{P}_\Theta \in \Delta(\mathcal{O})}  \inf_{\mathbb{P}_\Pi \in \Delta(\mathcal{P})}\mathbb{E}_\Theta \mathbb{E}_\Pi \left[\mathfrak{R}_{\mathcal{M}}(\Pi, \Theta) \right] \leq K_2.
\end{split}
\]
This shows that, under minimax duality conditions, we can apply bounds on MBR to control the minimax regret.
The theorem presented below provides information-theoretic upper bounds on the minimax regret by using the minimum Bayesian regret bounds outlined in \cite[Section V]{gouverneur2022information} under diverse conditions of the reward function and the probability distributions of states and observations.
These bounds are derived using terms associated with a natural Bayesian reinforcement learning approach, specifically the Thompson sampling algorithm.
\begin{theorem} 
Consider the function \(f^\star\) that maximizes the expected utility, defined as

$f^\star = \arg\sup_{f: \mathcal{S} \times \mathcal{O} \rightarrow \mathcal{A}} \mathbb{E}\left[ \sum_{t=1}^T r(Y_t, f(S_t,\Theta)) \right].$
Let \(S_t^\star\) and \(Y_t^\star\) denote the states and observations at time \(t\) when the function \(f^\star\) is followed. Additionally, let \(\hat{S}_t\), \(\hat{Y}_t\), and \(\hat{H}_t\) represent the states, observations, and history at time \(t\) obtained using the Thompson sampling algorithm \cite{russo2016information}, \cite{chapelle2011empirical}. Then, we obtain the following upper bounds for the minimax regret:
    \begin{enumerate}    
        \item If for all $t = 1, \ldots, T$, the random reward obtained by following $f^\star$, $r(\hat{Y}_t, f^\star(\hat{S}_t, \theta))$ is $\sigma_t^2$-sub-Gaussian under $\mathbb{P}_{\hat{Y}_t, \hat{S}_t \mid \hat{H}_t = \hat{h}_t}$ for all $\theta \in \mathcal{O}$ and all $\hat{h}_t \in \mathcal{H}_t$, then 
        $$\mathfrak{M}_{\mathcal{M}}  \leq  \sup_{\mathbb{P}_\Theta} \sum_{t=1}^T\mathbb{E}\left[\sqrt{2\sigma_t^2 D_{KL}(\mathbb{P}_{Y_t^\star, S_t^\star \mid \Theta} \mid\mid \mathbb{P}_{\hat{Y}_t, \hat{S}_t \mid \hat{H}_t})}\right].$$
        \item Suppose that $(\mathcal{Y}\times\mathcal{A} )$ is a metric space with metric $\rho$.   If the reward function $r : \mathcal{Y} \times \mathcal{A} \to \mathbb{R}$ is $L$-Lipschitz under the metric $\rho$, then 
         $$\mathfrak{M}_{\mathcal{M}}  \leq \sup_{\mathbb{P}_\Theta} L \sum_{t=1}^T\mathbb{E}\left[W(\mathbb{P}_{Y_t^\star, S_t^\star \mid \Theta}, \mathbb{P}_{\hat{Y}_t, \hat{S}_t \mid \hat{H}_t})\right].
         $$
         where \( W(\cdot, \cdot) \) represents the Wasserstein distance.
    \end{enumerate}
\end{theorem}
Thus, we can find 
Furthermore, the minimax regret for specific problems can be upper bounded using specialized bounds as derived in \cite{gouverneur2022information}, \cite{gouverneur2023thompson}, and \cite{dong2018information}.

\subsection{Finite Multi-arm Bandit Problem  with bounded reward}
The finite Multi-Armed Bandit (MAB) problem with bounded rewards is a specific class of MDPs. Let $\mathcal{A}$ be finite. Formally, the finite MAB problem is defined by the class \(\mathcal{B} = (\mathcal{S}, \mathcal{A}, p,y,r, T)\), where \(\mathcal{S} = \{s\}\) and \(Y_t\) is independent of \(S_t\) given \(\Theta\), for all $t$. Moreover, let \(\mathcal{O}\) be a finite set.
Let us introduce a prior distribution \(\mathbb{P}_{\Theta}\). We can then apply~\Cref{corollary: minimax for finite}  as we have a finite set of MDPs with finite states and action spaces.

The bound derived in \cite[Corollary 5]{gouverneur2022information} can be used to establish an upper bound on the minimum Bayesian regret: $
\mathfrak{F}_{\mathcal{B}}(\mathbb{P}_\Theta)\ \leq \sup_{\mathbb{P}_{\Theta} \in \Delta(\mathcal{O})} \sqrt{\frac{1}{2}|A| H(A^\star)T)}
,$
and hence obtain the minimax regret bound:
$
 \mathfrak{M}_{\mathcal{B}} \leq O(\sqrt{|\mathcal{A}|\log|\mathcal{A}|T}),
$
which matches the upper bound obtained in \cite{foster2021statistical} and exceeds the lower bound only by a $\sqrt{\log|\mathcal{A}|}$ factor \cite{audibert2009minimax}.

\subsection{Linear Bandits}

Consider a linear bandit problem defined by \(\mathcal{L} = (\mathcal{S}, \mathcal{A}, p,y,r, T)\), where \(S = \{s\}\) and \(Y_t\) is independent of \(S\) given \(\Theta\). The actions are represented as a \(d\)-dimensional vector, i.e., \(\mathcal{A} \subset \mathbb{R}^d\), and the reward from an action \(a \in \mathcal{A}\) satisfies \(\mathbb{E}^\theta[r(Y_t, a)] = a^T\theta\). Furthermore, we assume the conditions from~\Cref{theorem: minimax for continuous} hold. In addition, consider a ball-structured action space and the parameter space \(\mathcal{A}, \mathcal{O} \subseteq \mathbb{B}^d(0,1)\), where \(\mathbb{B}^d(0,1)\) denotes the \(d\)-dimensional closed Euclidean unit ball. Under these conditions, an algorithm \(\hat{\pi}\)-specifically the Thompson sampling algorithm can achieve a regret bound \cite{dong2018information} given by: $
\mathfrak{BR}_{\mathcal{L}}(\hat{\pi}, \mathbb{P}_\Theta)\leq O(d \sqrt{T \log T}).$ This further implies that the MBR can be bounded above, i.e., $
\mathfrak{F}_{\mathcal{L}}(\mathbb{P}_\Theta)\leq  O(d \sqrt{T \log T}).  $
Consequently, under the conditions of the minimax theorem, the minimax regret is bounded by $
\mathfrak{M}_{\mathcal{L}}  \leq O(d \sqrt{T \log T}),
$
which exceeds the lower bound for this problem \cite{foster2021statistical} by only a $\sqrt{\log T}$ factor.

\subsection{Contextual Bandits}

Consider a contextual bandit problem defined by \(\mathcal{C} = (\mathcal{S}, \mathcal{A}, p,y,r, T)\). The transition kernel \(p\) is such that the next state at time \(t+1\) is independent of both the previous state and the action taken at time \(t\).
 We also assume the conditions from~\Cref{corollary: minimax for finite} hold. Assuming the rewards are bounded in \([0, 1]\), for any contextual bandit problem \(\mathcal{C}\), there exists an algorithm \(\hat{\pi}\) (specifically, the Thompson sampling algorithm) such that the Bayesian regret after \(T\) rounds is bounded as follows \cite{gouverneur2023thompson}: $
\mathfrak{BR}_{\mathcal{C}}(\hat{\pi}, \mathbb{P}_\Theta) \leq  \sqrt{\frac{|\mathcal{A}|T H(\Theta)}{2}}.
$ This further implies that the MBR can be bounded above, i.e., $
\mathfrak{F}_{\mathcal{C}}(\mathbb{P}_\Theta)\leq \sqrt{\frac{|\mathcal{A}|T H(\Theta)}{2}}.  $ Therefore, under the conditions of the minimax theorem, the minimax regret is bounded by $
\mathfrak{M}_{\mathcal{C}}  \leq O\left( \sqrt{|\mathcal{A}|T \log|\mathcal{O}|}\right),
$ which matches the optimal rate  \cite[Section 1.2]{foster2020beyond}.

\section{Conclusion}
\label{sec:conclusion}
This work demonstrates how Bayesian regret bounds can be applied to establish information-theoretic minimax regret bounds for RL problems in the form of Markov decision processes. We extend the minimax theorem to more general spaces beyond finite state and action spaces, deriving minimax regret bounds across various problem settings, including bandit, contextual bandit, and reinforcement learning problems. Our analysis recovers upper bounds for specific problem classes. Future research directions include exploring additional conditions for duality and investigating other suitable regret definitions, including risk functionals.
\newpage
\bibliographystyle{IEEEtran}
\balance
\bibliography{references}

\begin{thebibliography}{10}
\providecommand{\url}[1]{#1}
\csname url@samestyle\endcsname
\providecommand{\newblock}{\relax}
\providecommand{\bibinfo}[2]{#2}
\providecommand{\BIBentrySTDinterwordspacing}{\spaceskip=0pt\relax}
\providecommand{\BIBentryALTinterwordstretchfactor}{4}
\providecommand{\BIBentryALTinterwordspacing}{\spaceskip=\fontdimen2\font plus
\BIBentryALTinterwordstretchfactor\fontdimen3\font minus \fontdimen4\font\relax}
\providecommand{\BIBforeignlanguage}[2]{{%
\expandafter\ifx\csname l@#1\endcsname\relax
\typeout{** WARNING: IEEEtran.bst: No hyphenation pattern has been}%
\typeout{** loaded for the language `#1'. Using the pattern for}%
\typeout{** the default language instead.}%
\else
\language=\csname l@#1\endcsname
\fi
#2}}
\providecommand{\BIBdecl}{\relax}
\BIBdecl

\bibitem{moos2022robust}
J.~Moos, K.~Hansel, H.~Abdulsamad, S.~Stark, D.~Clever, and J.~Peters, ``Robust reinforcement learning: A review of foundations and recent advances,'' \emph{Machine Learning and Knowledge Extraction}, vol.~4, no.~1, pp. 276--315, 2022.

\bibitem{lattimore2020bandit}
T.~Lattimore and C.~Szepesv{\'a}ri, \emph{Bandit algorithms}.\hskip 1em plus 0.5em minus 0.4em\relax Cambridge University Press, 2020.

\bibitem{azar2017minimax}
M.~G. Azar, I.~Osband, and R.~Munos, ``Minimax regret bounds for reinforcement learning,'' in \emph{International conference on machine learning}.\hskip 1em plus 0.5em minus 0.4em\relax PMLR, 2017, pp. 263--272.

\bibitem{gyorfi2023lossless}
L.~Gy{\"o}rfi, T.~Linder, and H.~Walk, ``Lossless transformations and excess risk bounds in statistical inference,'' \emph{Entropy}, vol.~25, no.~10, p. 1394, 2023.

\bibitem{hafez2021rate}
H.~Hafez-Kolahi, B.~Moniri, S.~Kasaei, and M.~S. Baghshah, ``Rate-distortion analysis of minimum excess risk in bayesian learning,'' in \emph{International Conference on Machine Learning}.\hskip 1em plus 0.5em minus 0.4em\relax PMLR, 2021, pp. 3998--4007.

\bibitem{xu2022minimum}
A.~Xu and M.~Raginsky, ``Minimum excess risk in {Bayesian} learning,'' \emph{IEEE Transactions on Information Theory}, vol.~68, no.~12, pp. 7935--7955, 2022.

\bibitem{hafez2023information}
H.~Hafez-Kolahi, B.~Moniri, and S.~Kasaei, ``Information-theoretic analysis of minimax excess risk,'' \emph{IEEE Transactions on Information Theory}, 2023.

\bibitem{lattimore2019information}
T.~Lattimore and C.~Szepesv{\'a}ri, ``An information-theoretic approach to minimax regret in partial monitoring,'' in \emph{Conference on Learning Theory}.\hskip 1em plus 0.5em minus 0.4em\relax PMLR, 2019, pp. 2111--2139.

\bibitem{buening2023minimax}
T.~K. Buening, C.~Dimitrakakis, H.~Eriksson, D.~Grover, and E.~Jorge, ``Minimax-bayes reinforcement learning,'' in \emph{International Conference on Artificial Intelligence and Statistics}.\hskip 1em plus 0.5em minus 0.4em\relax PMLR, 2023, pp. 7511--7527.

\bibitem{gouverneur2022information}
A.~Gouverneur, B.~Rodr{\'\i}guez-G{\'a}lvez, T.~J. Oechtering, and M.~Skoglund, ``An information-theoretic analysis of {Bayesian} reinforcement learning,'' in \emph{2022 58th Annual Allerton Conference on Communication, Control, and Computing (Allerton)}.\hskip 1em plus 0.5em minus 0.4em\relax IEEE, 2022, pp. 1--7.

\bibitem{cesa2006prediction}
N.~Cesa-Bianchi and G.~Lugosi, \emph{Prediction, learning, and games}.\hskip 1em plus 0.5em minus 0.4em\relax Cambridge university press, 2006.

\bibitem{van2003probability}
O.~van Gaans, ``Probability measures on metric spaces,'' \emph{Lecture notes}, 2003.

\bibitem{orevkov1973equivalence}
V.~Orevkov, ``Equivalence of two definitions of continuity,'' \emph{Journal of Soviet Mathematics}, vol.~1, no.~1, pp. 92--99, 1973.

\bibitem{gray2009probability}
R.~M. Gray, \emph{Probability, random processes, and ergodic properties}.\hskip 1em plus 0.5em minus 0.4em\relax Springer Science \& Business Media, 2009.

\bibitem{russo2016information}
D.~Russo and B.~Van~Roy, ``An information-theoretic analysis of {Thompson} sampling,'' \emph{Journal of Machine Learning Research}, vol.~17, no.~68, pp. 1--30, 2016.

\bibitem{chapelle2011empirical}
O.~Chapelle and L.~Li, ``An empirical evaluation of {Thompson} sampling,'' \emph{Advances in neural information processing systems}, vol.~24, 2011.

\bibitem{gouverneur2023thompson}
A.~Gouverneur, B.~Rodr{\'\i}guez-G{\'a}lvez, T.~J. Oechtering, and M.~Skoglund, ``Thompson sampling regret bounds for contextual bandits with sub-gaussian rewards,'' in \emph{2023 IEEE International Symposium on Information Theory (ISIT)}.\hskip 1em plus 0.5em minus 0.4em\relax IEEE, 2023, pp. 1306--1311.

\bibitem{dong2018information}
S.~Dong and B.~Van~Roy, ``An information-theoretic analysis for {Thompson} sampling with many actions,'' \emph{Advances in Neural Information Processing Systems}, vol.~31, 2018.

\bibitem{foster2021statistical}
D.~J. Foster, S.~M. Kakade, J.~Qian, and A.~Rakhlin, ``The statistical complexity of interactive decision making,'' \emph{arXiv preprint arXiv:2112.13487}, 2021.

\bibitem{audibert2009minimax}
J.-Y. Audibert and S.~Bubeck, ``Minimax policies for adversarial and stochastic bandits,'' in \emph{COLT}, 2009, pp. 217--226.

\bibitem{foster2020beyond}
D.~Foster and A.~Rakhlin, ``Beyond {UCB}: Optimal and efficient contextual bandits with regression oracles,'' in \emph{International Conference on Machine Learning}.\hskip 1em plus 0.5em minus 0.4em\relax PMLR, 2020, pp. 3199--3210.

\end{thebibliography}

\end{document}